\documentclass[conference]{IEEEtran}

\usepackage{tikzsymbols}

\usepackage{mathpazo} 
\usepackage{subcaption}
\usepackage{graphicx}
\usepackage{booktabs}
\usepackage{verbatim}
\usepackage{amsmath,amssymb,amsthm}
\usepackage{mathtools}
\usepackage{algorithm}
\usepackage[noend]{algpseudocode}
\usepackage{url}
\usepackage{amsthm}
\usepackage{array}
\usepackage{flushend}
\usepackage{comment}

\newtheorem{mydef}{Definition}
\newtheorem{myquestion}{Question}
\newtheorem{theorem}{Theorem}
\newtheorem{lemma}{Lemma}

\newcommand{\eref}[1]{(\ref{#1})} 
\newcommand{\sref}[1]{Sec. \ref{#1}} 
\newcommand{\figref}[1]{Fig.\ref{#1}} 
\newcommand{\dref}[1]{Definition \ref{#1}} 
\newcommand{\qref}[1]{Question \ref{#1}} 

\newcommand{\xxnote}[3]{}
\ifx\hidenotes\undefined
  \usepackage{color}
  \renewcommand{\xxnote}[3]{\color{#2}{#1: #3}}
\fi

\usepackage[numbers]{natbib}
\usepackage{multicol}
\usepackage[hidelinks]{hyperref}

\begin{document}

\title{A New Paradigm for Robotic Dust Collection: Theorems, User Studies, and a Field Study}

\author{\authorblockN{Rachel Holladay}
\authorblockA{Robotics Institute \\
Carnegie Mellon University\\
rmh@andrew.cmu.edu}
\and
\authorblockN{Siddhartha S.~Srinivasa}
\authorblockA{Robotics Institute \\
Carnegie Mellon University\\
siddh@cs.cmu.edu}}

\maketitle

\begin{abstract}
We pioneer a new future in robotic dust collection by introducing passive dust-collecting robots that, unlike their predecessors, do not require locomotion to collect dust. While previous research has exclusively focused on active dust-collecting robots, we show that these robots fail with respect to practical and theoretical aspects, as well as human factors. 
By contrast, passive robots, through their unconstrained versatility, shine brilliantly in all three metrics. 
We present a mathematical formalism of both paradigms followed by a user study and field study. 
\end{abstract}

\IEEEpeerreviewmaketitle

\section{Introduction}
\label{sec:introduction}
There has been renewed recent interest in the design of efficient and robust dust-collecting robots~\cite{prassler2000short,fiorini2000cleaning}. 
The oppression of constant dust raining over our heads calls out for immediate attention. Furthermore, the increased cost
of legal human labor, and increased penalties for employing illegal immigrants, has made dust-collection all the more critical to automate~\cite{jones2006robots,tribelhorn2007evaluating}.

However, all of the robotic solutions have focussed exclusively on what we define (see \dref{def:active_robot} for a precise mathematical definition) as \emph{active dust-collecting robots}. Informally, these are traditional robotic solutions, where the robot locomotes to collect dust. It is understandable why this seems like a natural choice as humans equipped with vacuum cleaners are, after all, also active dust-collectors.

Unfortunately, active dust collection presents several challenges: (1) \emph{Practical:} they require locomotion, which requires motors and wheels, which are expensive and subject to much wear and tear, (2) \emph{Theoretical:} most active dust-collectors are wheeled robots, which are subject to nonholonomic constraints on motion, demanding complex nonlinear control even for seemingly simple motions like moving sideways~\cite{doh2007practical,choset2001coverage,ulrich1997autonomous}, (3) \emph{Human factors:} several of our users in our user study expressed disgust, skepticism, and sometimes terror, about the prospect of sentient robots wandering around their homes, for example:
\begin{quote}
I don't want a f*cking robot running around all day in my house.
\end{quote}

In this paper, we propose a completely new paradigm for dust collection: \emph{passive dust-collecting robots} (see \dref{def:passive_robot} for a precise mathematical definition). Informally, these are revolutionary new solutions that are able to collect dust \emph{without any locomotion}!

As a consequence, passive dust-collecting robots address all of the above challenges: (1) \emph{Practical:} Because they have no moving parts like wheels or motors, they are both inexpensive and incur no wear and tear, (2) \emph{Theoretical:} because passive dust-collectors can be trivially parallel transported to the identity element of the $\mathbb{SE}(2)$ Lie Group, they require no explicit motion planning (in situations where parallel transport is inefficient, the robot can be physically transported to the identity element), (3) \emph{Human Factors:} as passive dust-collecting robots are identical to other passive elements in our homes and work places (like walls, tables, desks, lamps, carpets), their adoption into our lifespace is seamless. 

\begin{table}[t!]
\centering
\begin{tabular}{>{\centering\arraybackslash}m{0.5in}>{\centering\arraybackslash}m{0.5in}>{\centering\arraybackslash}m{0.5in}}
\toprule
 & Active & Passive \\
 \midrule
Practical & \Sadey[2][red] & \Smiley[2][yellow]\\
Theoretical & \Sadey[2][red] & \Smiley[2][yellow]\\
Human Factors & \Sadey[2][red] & \Smiley[2][yellow]\\
\bottomrule
\end{tabular}
\caption{Passive dust-collection outperforms active dust-collection in all metrics}
\end{table}

In addition, we present and analyze a mathematical model of dust collection. Using our model, we can, for the first time, answer which robot-type is more efficient. This is a critical question to consider in order to inform future cleaning choices. 

Our analysis reveals that for a certain choice of constants, a passive dust cleaning robot is more efficient than its active counterpart. Through a user study, we contrast this with user's perceived perception of robot efficiency and what factors influence their choices. 

To explore what choices are actually made we leveraged a field study of Carnegie Mellon's Robotics Institute to determine the prevalence of each robot type. This study reveals that passive dust collecting allows for a wider range of morphologies, suggesting that passive dust collecting is a more inclusive characterization. Furthermore, we see that rather than two paradigms there is a continuum of dust collecting robots. 

Our work makes the following contributions: 

\textbf{Mathematical Formulation.} We present a model of active and passive dust collecting robots followed by an efficiency tradeoff analysis. 

\textbf{Preference User Study.} We surveyed college students to determine what kind of robot they preferred and which they perceived to be more efficient. 

\textbf{Field Study.} Using data on the robots of the Robotics Institute we investigate the more popular robotic paradigms. 

We believe our work takes a first step in launching a new discussion concerning the nature of robotic dust collection, paving the way for future cleanliness. 

\section{A Mathematical Model for Dust Collection}
\label{sec:dust_model}

In order to compare and analyze active and passive dust collecting robots we present a mathematical model of their dust collection capabilities. With this model, we dare to ask: which robot is more efficient? 

\subsection{Dust Model}

We model dust as a pressureless perfect fluid, which has a positive mass density but vanishing pressure. Under this assumption, we can model the interaction of dust particles by solving the Einstein field equation, whose stress-energy tensor can be written in this simple and elegant form
\begin{equation}
T^{{\mu \nu }}=\rho U^{\mu }U^{\nu }
\end{equation}
where the world lines of the dust particles are the integral curves of the four-velocity $U^{\mu }$,
and the matter density is given by the scalar function $\rho$. 

Remarkably, unless subjected to cosmological radiation of a nearby black hole, or a near-relativistic photonic Mach cone, this equation can be solved analytically, resulting in dust falling at a constant rate of $\alpha$.

We model our robots as covering 1 unit$^{2}$ area of space-time. We present our models for passive and active robots before performing comparative analysis.  

\subsection{Pasive Robot Model}
\label{sec:passive_model}
We provide the following formalism:
\begin{mydef}
We define a \underline{passive dust collecting robot} as a robot that does not move, collecting the dust that falls upon it.
\label{def:passive_robot}
\end{mydef}

\begin{lemma}
The dust-collecting capability of a passive dust-collecting robot is given by
\begin{equation}
D_{passive} \coloneqq \alpha
\label{eqn:passive}
\end{equation}
\end{lemma}

\subsection{Active Robot Model}
\label{sec:active_model}

We provide the following formalism: 
\begin{mydef}
We define an \underline{active dust collecting robot} as a robot that moves around the space, actively collecting dust.
\label{def:active_robot}
\end{mydef}

We model our active robot as driving at speed $\beta$. We assume that our robot can only active collect dust of height $h$. 
This assumption is drawn from IRobot's Roomba, which reportly can get stuck on cords and cables. 
As a simplifying assumption we will assume that the robot always collects dust of height $h$, implying that there is always at least dust of height $h$ prior to the robot's operation. 

\begin{lemma}
The dust-collecting capability of an active dust-collecting robot is given by
\begin{equation}
D_{active} \coloneqq h\beta^{3} + \frac{\alpha}{\beta}
\label{eqn:active}
\end{equation}
\end{lemma}

\begin{proof}
It is obvious that the robot actively collects $h\beta^{3}$ dust. 

However this is not the entire story. 
As the robot drives, actively collecting dust, it also passively collects the dust that happens to fall on it. 
To model this, we consider the robot passing over some fixed line. 
Some portion of the robot is occulding this line for $\frac{1}{\beta}$ seconds. Thus the robot passively collects $\frac{\alpha}{\beta}$ dust. 

Thus, combining the active and passive components our active robot collects:
\begin{equation*}
D_{active} \coloneqq h\beta^{3} + \frac{\alpha}{\beta}
\label{eqn:active}
\end{equation*}
\end{proof}

\subsection{Model Comparison}
\label{sec:model_comparison}

We next compare for what tradeoffs there are between passive and active dust cleaning robots. We pose this as the question: When are passive dust cleaning robots more efficient then their active counterparts? Hence when is $D_{passive} > D_{active}$?

We are now ready to prove our main theorem.
\begin{theorem}
The dust-collecting capability of a passive robot exceeds the dust-collecting capability of an active robot when
\begin{equation}
\alpha > \frac{h\beta^{4}}{\beta-1}
\end{equation}
\end{theorem}
\begin{proof}
Using \eref{eqn:passive} and \eref{eqn:active} we get:
\begin{align*}
D_{passive} &> D_{active}\\
\alpha &> h\beta^{3} + \frac{\alpha}{\beta}
\end{align*}

With some simple arithmetic this becomes:
\begin{equation*}
\alpha > \frac{h\beta^{4}}{\beta-1}
\end{equation*}
\end{proof}

\figref{fig:model_comparison} shows this function over a variety of $\beta$s and a few choices of $h$. The y-axis can be viewed as a measure of efficiency. 
A passive robot's efficiency corresponds to a straight line across the y-axis at its $\alpha$ value.

As the $h$ value increases, the active robot's efficiency increases, which follows from the fact that as it drives, it can collect more dust. While we see an initial drop in efficiency due to a $\beta$ increase, owing to the fact that the active robot collects less dust passively, this effect is then dwarfed by a faster moving robot that can cover more ground.

\begin{figure}[t!]
\centering
\includegraphics[width=.99\columnwidth]{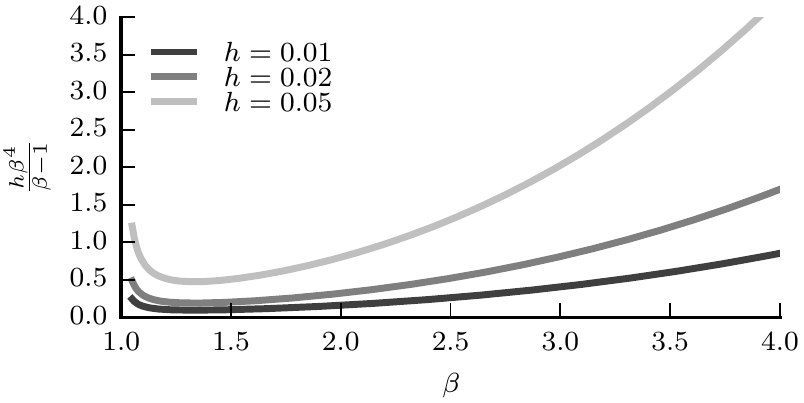}
\caption{Comparing the efficiency of each robot type as a function of $\beta$.}
\label{fig:model_comparison}
\end{figure}

\section{User Study}
\label{sec:user_study}

Having a developed a model of passive and active dust collecting robots we used a user study to evaluate people's opinions on each type of robot\'s efficiency.
This is critical in developing effective robots as we need to explore the possible discrepencies between perceived versus actual robot capability~\cite{cha2015perceived}.

\subsection{Experimental Setup}
\label{sec:experimental_setup}

We created an online form to evaluate users opinions of passive and active dust collecting robots. 
Provided users with \dref{def:passive_robot} and \dref{def:active_robot}, we then asked them the following questions:
 
\begin{myquestion}
Which type of robot do you think collects more dust: an active dust collecting robot or a passive dust collecting robot? Why?
\label{q:collects_more}
\end{myquestion}

\begin{myquestion}
Which robot would you prefer to have?
\label{q:prefer}
\end{myquestion}

For \qref{q:prefer} the options were: Active dust collecting robot, Passive dust collecting robot, Whichever robot is the most efficient at collecting dust. Our goal in asking this was to determine what people value more, the illusion of efficiency or actual efficiency. 

\textbf{Participants} We recurited 23 Carnegie Mellon students (14 males, 9 females, aged 21-23) through online sources. 

\subsection{Analysis}
\label{sec:study_analysis}
The results of our user study can be seen in \figref{fig:user_results}. While people believe that the active robot collects more dust, people would prefer to have the most efficient robot, regardless of its capabilities. 

\begin{figure}[t!]
\centering
\includegraphics[width=.99\columnwidth]{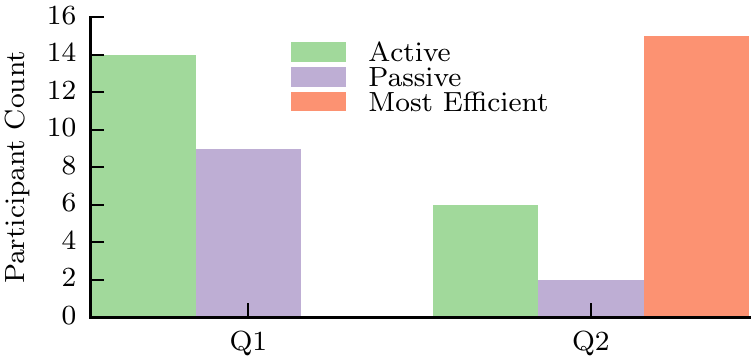}
\caption{User Study Results}
\label{fig:user_results}
\end{figure}

What is perhaps more telling is the variety of user responses we had to why they believed each robot would collect more dust. 

Those who supported passive dust collecting robots listed a variety of reasons, with many people concerns with active dusting robots dispersing and upsetting more dust than the collect.
One user rationalized his choice by the nature of dust saying "I've observed that the stuff that collects the most dust in my place are the items that are static, therefore I would assume that the static robot might collect more dust."

Still other users took a more global view with one user, as mentioned above, claiming that they "don't want a f*cking robot running around all day" and another, acccepting the harsh realities of time remarked "All robots ultimately become a passive dust-collecting robot."

For every supporter of passive robots, there were still more who argued for active robots. Almost every person, in explaining their choice, argued that active robots, due to their mobility, would be able to cover a larger space. This highlights the dichotomy between \textit{efficiency and coverage}. 

While our passive dust collecting robot can provide superior efficiency, its lack of locomotion greatly reduces is potential coverage. By constrast, the active robot has the ability to move around, coverage potentially all of the room, given some amount of time. 

\section{Field Study}
\label{sec:field_study}

\begin{figure*}[t!]
\centering
\includegraphics[width=0.99\textwidth]{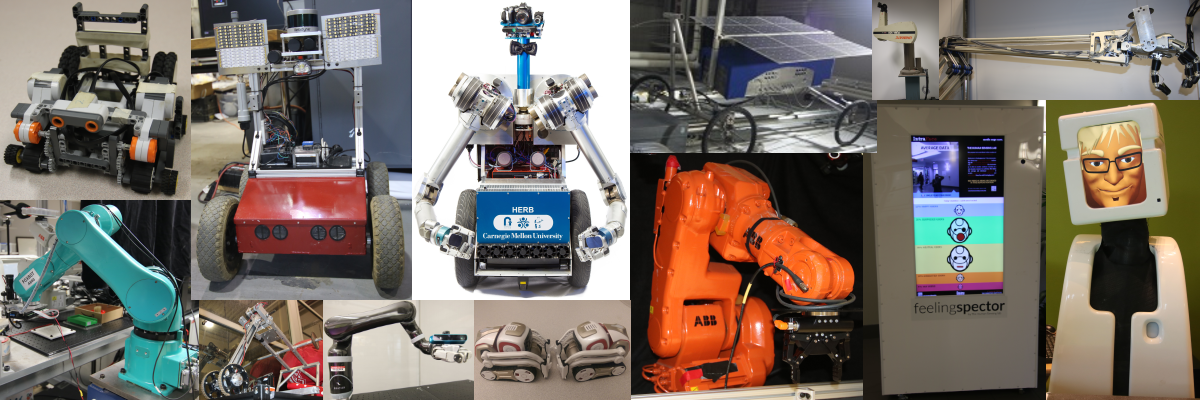}
\caption{Throughout the robots that can be seen at Carnegie Mellon's Robotics Institute we see a variety of passive dust collecting robots that range widely in shape, size, initial function and even cost.}
\label{fig:robot_montage}
\end{figure*}

Given the results of our user study in \sref{sec:user_study}, we next probe into how these preferences are reflected in reality. 
Carnegie Mellon's Robotics Institute is home is a large variety of robots and using the 2010 robot census we analyzed what kind of dust collecting robot we actually see~\cite{robot_census}. 

Of the 261 robots listed on the census\footnote{The original census data was provided directly from its author.} with complete information, we see that none of them are designed to collect dust actively. However, we can assume many of them collect dust passively. 
Twenty were listed as having no mobility, making them official passive dust collecting robots.
Even the eighty-six robots that have wheeled mobility are unlikey to be driving most of the time and therefore spend much of their life as passive dust collecting robots.  

In fact, broadening out, despite the variety in morphologies and mobilities from wheeled to winged, from manipulation to entertainment to competition, most, if not all, of the robots at the Robotics Institute spend large quantities of their tenure as passive dust collecting robots. 
While active dust collecting robots are constrained by their function to have certain properties, passive dust collecting robotics is an all-inclusive, all-accepting genre that allows for nearly any charactertization. 
We see a huge variety of robots in \figref{fig:robot_montage}.
They can be old or new, outrageously expensive or dirt cheap, beautifully crafted or hastily thrown together. 

Yet, if they do nothing, they all have the ability inside of them to be passive dust collecting robots. 
Given the guidelines provided by our model in \sref{sec:dust_model}, these robots have the capacity to be more efficient than their try-hard active collection counterparts.  
Based on the results of our study (\sref{sec:user_study}) this makes them more desirable.  

From these insights, it is now clear why the CMU Robotics Institute does not have any active dust collecting robots on record. 
They have been surpased by their more efficient, more inclusive, more desirable counterparts: passive dust collecting robots. 

\section{Discussion}
\label{sec:discussion}

While our analysis presented in \sref{sec:dust_model} outlines two classes of robots, our field study from \sref{sec:field_study} reveals a continuum of dust collecting robots.
Robots that do not active collect dust but are not entirely stationary, such as robots that are simply underused, represent the middle ground of dust collection.  
We can even think of air filters as dust collecting robots that actively collect dust but do not do so by moving themselves. This adds a new dimension of what it means for a robot to be active.

This work also aims to highlight the underappreciated advantages of passive dust collecting robots.
Passive robots, unconstrained by a need for explicit dust collecting capabilities, afford a wide range of mophologies. 
This allows for incredibly flexibility in designing the possible human-robot interaction schemes, which is critical to a cleaning robot\'s acceptance~\cite{hendriks2011robot, forlizzi2006service}.

While we focused on dust collecting robots are model generalizes to other situations, such as moving in the rain. Specifically, our model can be used to model whether you would get more wet by standing still or running through the rain. 

We hope that this work will raise awareness for passive dust collecting robots and raise further discussion on the nature of dust collection.

\section*{Acknowledgments}
This material is based upon work supported by the infinite discretionary money-bag.
We do not thank the members of the Personal Robotics Lab for helpful discussion and advice as this project was kept entirely super secret from them.

\bibliographystyle{plainnat}

\end{document}